\documentclass[letterpaper, 10 pt, conference]{ieeeconf}  

\usepackage{amsmath,amssymb,amsfonts}

\usepackage{algorithm}
\usepackage[T1]{fontenc}
\usepackage{booktabs}

\usepackage{lineno}
\modulolinenumbers[5]

\usepackage[font=small,labelfont=bf]{caption}
\usepackage{subcaption}

\usepackage{algpseudocode}

\usepackage{graphicx}
\usepackage{textcomp}
\usepackage{url}
\usepackage{color}

\usepackage{lipsum}
\usepackage{epstopdf}

\usepackage{amsopn}
\DeclareMathOperator{\diag}{diag}

\usepackage{amsthm}

\usepackage{enumerate}
\usepackage{cancel}
\usepackage{mathrsfs}
\usepackage{mathdots}
\usepackage{euscript}
\usepackage{amscd}
\usepackage{placeins}

\usepackage{tikz}
\usetikzlibrary{arrows}
\usepackage{verbatim}

\newtheorem{theorem}{Theorem}

\newtheorem{lemma}{Lemma}

\newtheorem{assumption}{Assumption}

\theoremstyle{definition}
\newtheorem{definition}{Definition}

\theoremstyle{remark}
\newtheorem{remark}{Remark}

\newcommand{\bmat}{\left[ \begin{matrix}}
	\newcommand{\emat}{\end{matrix} \right]}
\newcommand{\innerprod}[2]{\langle{#1},\,{#2}\rangle}

\DeclareMathOperator{\argmin}{argmin}
\DeclareMathOperator{\prox}{prox}

\newcommand{\Rbb}{\mathbb R}
\newcommand{\Hbb}{\mathbb H}
\newcommand{\Nbb}{\mathbb N}

\newcommand{\Fbb}{\mathbb F}
\newcommand{\xb}{\mathbf  x}
\newcommand{\yb}{\mathbf  y}
\newcommand{\sbf}{\mathbf  s}  
\newcommand{\zb}{\mathbf  z}
\newcommand{\wb}{\mathbf  w}
\newcommand{\vb}{\mathbf  v}

\newcommand{\fb}{\mathbf  f}

\newcommand{\db}{\mathbf  d}

\newcommand{\ub}{\mathbf  u}
\newcommand{\rb}{\mathbf  r}

\newcommand{\oneb}{\mathbf 1}
\newcommand{\zerob}{\mathbf 0}

\newcommand{\rhob}{\boldsymbol{\rho}}
\newcommand{\thetab}{\boldsymbol{\theta}}

\newcommand{\lambdab}{\boldsymbol{\lambda}}

\newcommand{\tol}{\mathtt{tol}}

\newcommand{\Lcal}{\mathcal{L}}
\newcommand{\Ical}{\mathcal{I}}
\newcommand{\Kcal}{\mathcal{K}}

\newcommand{\Xcal}{\mathcal{X}}

\newcommand{\test}{\text{test}}

\DeclareMathOperator{\sign}{sign}
\newcommand{\stepfunc}{\|(\cdot)_+\|_0}

\newcommand{\LzeroneSVM}{$L_{0/1}$-SVM}

\IEEEoverridecommandlockouts                              
\title{\LARGE \bf
An ADMM Solver for the MKL-$L_{0/1}$-SVM
}

\author{Yijie Shi and Bin Zhu
\thanks{This work was supported in part by Shenzhen Science and Technology Program (Grant No.~202206193000001-20220817184157001), the Fundamental Research Funds for the Central Universities,  and the ``Hundred-Talent Program'' of Sun Yat-sen University.}
\thanks{The authors are with School of Intelligent Systems Engineering, Sun Yat-sen University, Gongchang Road 66, 518107 Shenzhen, China. Emails: {\tt\small shiyj27@mail2.sysu.edu.cn} (Y. Shi), {\tt\small zhub26@mail.sysu.edu.cn} (B. Zhu).}
}

\begin{document}

\maketitle
\thispagestyle{empty}
\pagestyle{empty}

\begin{abstract}
We formulate the Multiple Kernel Learning (abbreviated as MKL) problem for the support vector machine with the infamous $(0,1)$-loss function. Some first-order optimality conditions are given and then exploited to develop a fast ADMM solver for the nonconvex and nonsmooth optimization problem. A simple numerical experiment on synthetic planar data shows that our MKL-\LzeroneSVM\ framework could be promising.
\end{abstract}




\section{Introduction}\label{sec:intro}

The support vector machine (SVM) is a classic tool in machine learning \cite{theodoridis2020machine}.
The idea dates back to the famous work of Cortes and Vapnik \cite{cortes1995support}. On p.~281 of that paper, the authors suggested (implicitly) the $(0,1)$-loss function, also called $L_{0/1}$ loss in \cite{wang2021support}, for quantifying the error of classification which essentially counts the number of samples to which the classifier assigns wrong labels. However, they also pointed out that the resulting optimization problem with the $(0,1)$ loss is \emph{NP-complete, nonsmooth, and nonconvex}, which directed researchers to the path of designing other (easier) loss functions, notably convex ones like the \emph{hinge} loss.
Recently in the literature, there is a resurging interest in the original SVM problem with the $(0,1)$ loss, abbreviated as ``$L_{0/1}$-SVM'', following theoretical and algorithmic developments for optimization problems with the ``\emph{$\ell_0$-norm}'', see e.g., \cite{nikolova2013description} and the references therein. In particular, \cite{wang2021support} proposed KKT-like optimality conditions for the $L_{0/1}$-SVM optimization problem and a efficient ADMM solver to obtain an \emph{approximate} solution.

In this work, we draw inspiration from the aforementioned papers and present a \emph{kernelized} version of the theory in which the ambient functional space has a richer structure than the usual Euclidean space. More precisely, we shall formulate the $L_{0/1}$-SVM problem in the context of \emph{Multiple Kernel Learning} \cite{rakotomamonjy2008simplemkl} and describe a first-order optimality theory as well as a numerical procedure for the optimization problem via the ADMM. Obviously, the MKL framework can offer much more flexibility than the single-kernel formulation by letting the optimization algorithm determine the best combination of different kernel functions. In this sense, our results represent a substantial generalization of the work in \cite{wang2021support} while maintaining the core features of \LzeroneSVM.


The remainder of this paper is organized as follows. Section~\ref{sec:prob} reviews the classic \LzeroneSVM\ in the single-kernel case, while Section~\ref{sec:MKL} discusses the MKL framework. Section~\ref{sec:optimality} establishes the optimality theory for the MKL-\LzeroneSVM\ problem,
and in Section~\ref{sec:algorithm} we propose an ADMM algorithm to solve the optimization problem. Finally, numerical experiments and concluding remarks are provided in Sections~\ref{sec:sims} and \ref{sec:conclu}, respectively.

\subsection*{Notation}

$\Rbb_+$ denotes the set of nonnegative reals, and $\Rbb_+^n :=\Rbb_+\times\cdots\times \Rbb_+$ the $n$-fold Cartesian product. $\Nbb_m:=\{1, 2, \dots, m\}$ is a finite index set for the data points and $\Nbb_L$ for the kernels.
Throughout the paper, the summation variables $i\in\Nbb_m$ is reserved for the data index, and $\ell\in\Nbb_L$ for the kernel index. We write $\sum_{i}$ and $\sum_{\ell}$ in place of $\sum_{i=1}^m$ and $\sum_{\ell=1}^L$ to simplify the notation.

\section{Problem formulation: the single-kernel case}\label{sec:prob}

Given the data set $\{(\xb_i,y_i) : i\in\Nbb_m\}$ where $\xb_i\in\Rbb^n$ and $y_i\in\{-1, 1\}$ the label, the binary classification task aims to predict the correct label $y$ for each vector $\xb$, seen or unseen. To this end, the SVM first lifts the problem to a \emph{reproducing kernel Hilbert space}\footnote{The theory of RKHS goes back to \cite{aronszajn1950theory} and many more, see e.g., \cite{paulsen2016introduction}. It has been used in the SVM as early as \cite{cortes1995support}.} (RKHS) $\Hbb$, in general infinite-dimensional and equipped with a \emph{positive definite} kernel function, say $\kappa : \Rbb^n \times \Rbb^n \to \Rbb$, via the feature mapping:
\begin{equation}
\xb \mapsto \phi(\xb) := \kappa(\cdot,\xb) \in \Hbb,
\end{equation}
and then considers discriminant (or decision) functions of the form
\begin{equation}\label{discrimi_func}
\tilde{f}(\xb) = b + \innerprod{w}{\phi(\xb)}_{\Hbb} = b+w(\xb),
\end{equation}
where $b\in\Rbb$, $w\in\Hbb$, $\innerprod{\cdot}{\cdot}_{\Hbb}$ the inner product associated to the RKHS $\Hbb$, and the second equality is due to the so-called \emph{reproducing property}. 
Note that such a discriminant function is in general nonlinear in $\xb$, but is indeed linear with respect to $\phi(\xb)$ in the feature space $\Hbb$.
The label of $\xb$ is assigned via $y(\xb)=\sign [\tilde{f}(\xb)]$ where $\sign(\cdot)$ is the sign function which gives $+1$ for a positive number, $-1$ for a negative number, and left undefined at zero.


In order to estimate the unknown quantities $b$ and $w$ in \eqref{discrimi_func}, one sets up the unconstrained optimization problem:
\begin{equation}\label{optim_infinit_dim}
\min_{\substack{w\in\Hbb, \, b\in\Rbb, \\ \tilde f(\cdot) = w(\cdot) + b}}\quad \frac{1}{2}\|w\|_{\Hbb}^2 + C \sum_{i} \Lcal(y_i, \tilde f(\xb_i)),
\end{equation}
where $\|w\|_{\Hbb}^2 = \innerprod{w}{w}_{\Hbb}$ is the squared norm of $w$ induced by the inner product, $\Lcal(\cdot,\cdot)$ is a suitable loss function, and $C>0$ is a regularization parameter balancing the two parts in the objective function. In the classic case where $\Hbb$ can be identified as $\Rbb^n$ itself, $\|w\|_{\Hbb}$ reduces to the Euclidean norm $\|\wb\|$ with $\wb=[w_1,\dots,w_n]^\top$. Moreover, the quantity $1/\|\wb\|$ can be interpreted as the width of the \emph{margin} between the decision hyperplane (corresponding to the equation $\wb^\top \xb +b = 0$) and the nearest points in each class, so that minimizing $\|\wb\|^2$ is equivalent to maximizing the margin width, an intuitive measure of robustness of the classifier. As for the loss function, we adopt the most natural choice:
\begin{equation}
\Lcal_{0/1} (y, \tilde{f}(\xb)) := H(1-y\tilde{f}(\xb))
\end{equation}
where $H$ is the (Heaviside) unit step function
\begin{equation}\label{func_unit_step}
H(t) = \begin{cases}
1, & t>0\\
0, & t\leq 0,
\end{cases}
\end{equation}
see also Fig.~\ref{fig:step_func}.

\begin{figure}
	\centering
	\begin{tikzpicture}
	\draw[line width=0.5pt][->](-2,0)--(2,0)node[left,below,font=\tiny]{$t$};
	\draw[line width=0.5pt][->](0,-0.8)--(0,1.4)node[right,font=\tiny]at(-0.8,1){$H(t)$};
	\node[below,font=\tiny] at (0.15,0){0};
	\node[right,below,font=\tiny]at(0.14,1){1};
	\draw[color=blue, thick,smooth,domain=0:1.8]plot(\x,1);
	\draw[color=blue, thick,smooth,domain=-1.8:0.0]plot(\x,0);
	\draw[color=blue,fill=blue,smooth]circle(0.03);
	\draw[fill = white](0,1)circle(0.03);
	\end{tikzpicture}
	\caption{The unit step function.}
	\label{fig:step_func}
\end{figure}
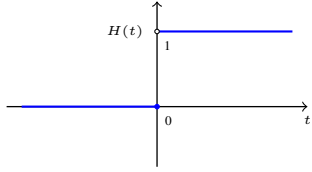

In order to understand the loss function, notice that in the case where two classes of points are \emph{linearly separable}, one can always identify a subset of decision hyperplanes such that $y_i f(\xb_i)\geq 1$ for all $i\in\Nbb_m$ \cite{vapnik2000nature}. In the \emph{linearly inseparable} case, however, the inequality can be violated by some data points and such violations are in turn penalized since the loss function now reads as
\begin{equation}
\Lcal_{0/1} (y, \tilde{f}(\xb)) = 
\begin{cases}
1, & \text{if}\ 1-y\tilde{f}(\xb) >0\\
0, & \text{if}\ 1-y\tilde{f}(\xb) \leq 0.
\end{cases}
\end{equation}
It is this latter case that will be the focus of this paper.

The optimization problem \eqref{optim_infinit_dim} is infinite-dimensional in general due to the ambient space $\Hbb$. It can however, be reduced to a \emph{finite-dimensional} one via the celebrated \emph{representer theorem} \cite{kimeldorf1971some}. More precisely, by the \emph{semiparametric} representer theorem \cite{scholkopf-smola2001earning}, any minimizer of \eqref{optim_infinit_dim} must have the form
\begin{equation}\label{discrimi_func_parametric}
\tilde f(\cdot) = \sum_{i} w_i \kappa(\,\cdot\,, \xb_i) + b,
\end{equation}
so that the desired function $w(\cdot)$ is completely characterized by the linear combination of the \emph{kernel sections} $\kappa(\,\cdot\,, \xb_i)$, and the coefficients in $\wb=[w_1,\dots,w_m]^\top$ become the new unknowns. After some algebra involving the \emph{kernel trick}, we are left with the following finite-dimensional optimization problem:
\begin{equation}\label{optim_finit_dim}
\min_{\substack{\wb\in\Rbb^m, \, b\in\Rbb}}\ \ J(\wb,b) := \frac{1}{2} \wb^\top K \wb + C \|(\oneb - A\wb -b\yb)_+\|_0
\end{equation}
where,
\begin{itemize}
	\item $K=K^\top$ is the \emph{kernel matrix}
	\begin{equation}
	\bmat \kappa(\xb_1, \xb_1) & \cdots & \kappa(\xb_1,\xb_m)\\
	\vdots & \ddots & \vdots \\
	\kappa(\xb_m,\xb_1) & \cdots & \kappa(\xb_m,\xb_m) \emat \in \Rbb^{m\times m}
	\end{equation}
	which is \emph{positive semidefinite} by construction,
	\item $\oneb\in\Rbb^m$ is a vector whose components are all $1$'s, 
	\item $\yb = [y_1,\dots, y_m]^\top$ is the vector of labels, 
	\item the matrix $A = D_{\yb} K$ is such that $D_{\yb} = \diag(\yb)$ is the diagonal matrix whose $(i,i)$ entry is $y_i$, 
	\item the function $t_+ := \max\{0, t\}$ takes the positive part of the argument when applied to a scalar\footnote{It is known as the ReLU (Rectified Linear Unit) activation function in the context of artificial neural networks.}, and $\vb_+ := [(v_1)_+, \dots, (v_m)_+]^\top$ represents componentwise application of the scalar function,
	\item $\|\vb\|_0$ is the $\ell_0$-norm\footnote{The term ``norm'' is abused here since strictly speaking, ``$\ell^p$-norms'' are not \emph{bona fide} norms for $0\leq p<1$ due to the violation of the triangle inequality.} that counts the number of nonzero components in the vector $\vb$.
\end{itemize}
Clearly, the composite function $\|\vb_+\|_0$ counts the number of (strictly) positive components in $\vb$. For a scalar $t$, it coincides with the step function in \eqref{func_unit_step}.

\begin{remark}\label{rem_redu}
	
	The above formulation includes the problem investigated in \cite{wang2021support} as a special case. To see this, consider the \emph{homogeneous polynomial} kernel
	\begin{equation}
	\kappa(\xb, \yb) = (\xb^\top \yb)^d
	\end{equation}
	with the degree parameter $d=1$. Then the discriminant function in \eqref{discrimi_func_parametric} becomes
	\begin{equation}
	\tilde{f}(\xb) = \sum_{i} w_i \xb_i^\top \xb + b = \tilde{\wb}^\top \xb +b,
	\end{equation}
	where $\tilde{\wb} := \sum_{i} w_i \xb_i\in\Rbb^n$ is identified as the new variable for optimization. Moreover, it is not difficult to verify the relation $\wb^\top K \wb = \tilde{\wb}^\top \tilde{\wb} = \|\tilde{\wb}\|^2$, so that the optimization problem in \cite{wang2021support} results. 
\end{remark}


For reasons discussed in Remark~\ref{rem_redu}, in the remaining part of this paper, we shall always assume that the kernel matrix $K$ is \emph{positive definite}, which is indeed true for the \emph{Gaussian} kernel 
\begin{equation}\label{Gaussian_kernel}
\kappa(\xb,\yb) = \exp \left( -\frac{\|\xb-\yb\|^2}{2\sigma^2}\right),
\end{equation}
where $\sigma>0$ is a parameter (called \emph{hyperparameter}), see \cite{slavakis2014online}.
In such a case, the matrix $A=D_{\yb} K$ in \eqref{optim_finit_dim} is also \emph{invertible} since $D_{\yb}$ is a diagonal matrix\footnote{In fact, $D_{\yb}$ is both \emph{involutory} and \emph{orthogonal}, i.e., $D_{\yb}^2 = D_{\yb}^\top D_{\yb}=I$.} whose diagonal entries are the labels $-1$ or $1$.



\section{Problem formulation: the multiple-kernel case}\label{sec:MKL}

In all kernel-based methods, the selection of a suitable kernel and its parameter is a major issue. Usually, this is done via cross-validation which inevitably has an ad-hoc flavor. An active research area to handle such an issue is called Multiple Kernel Learning (MKL), where one employs a set of different kernels and let the optimization procedure determine the proper combination. One possibility in this direction is to consider the nonlinear modeling function as follows:
\begin{equation}\label{multi-kernel_representer}
\begin{aligned}
\tilde f(\xb) & = \sum_{\ell} f_\ell(\xb) + b \\
& = \sum_{\ell} d_\ell \sum_{i} w_{i} \kappa_\ell(\xb,\xb_i) + b,
\end{aligned}
\end{equation}
where for each $\ell\in\Nbb_L$, $f_\ell$ lives in a different RKHS $\Hbb_\ell'$ corresponding to the kernel function $d_\ell \kappa_\ell(\cdot,\cdot)$, the parameters $d_\ell, b, w_{i}\in\Rbb$, and $\xb_i$ comes from the training data. In other words, the decision function $\tilde{f}$ is parametrized by $(\wb,\db,b)\in \Rbb^{m+L+1}$. 
In order to formally state our MKL optimization problem for the \LzeroneSVM, we need to borrow the functional space setup from \cite{rakotomamonjy2008simplemkl}.



For each $\ell\in\Nbb_L$, let $\Hbb_\ell$ be a RKHS of functions on $\Xcal\subset \Rbb^n$ with the kernel $\kappa_\ell(\cdot,\cdot)$ and the inner product $\innerprod{\cdot}{\cdot}_{\Hbb_\ell}$. Moreover, take $d_\ell\in\Rbb_+$, and define a Hilbert space $\Hbb_\ell'\subset \Hbb_\ell$ as
\begin{equation}
\mathbb{H}_\ell^\prime :=\left\{ f\in \mathbb{H}_\ell : \frac{\|f\|_{\mathbb{H}_\ell}}{d_\ell}<\infty\right\}
\end{equation}
endowed with the inner product
\begin{equation}
\left< f, g \right>_{\mathbb{H}_\ell^\prime}=\frac{\left< f, g \right>_{\Hbb_\ell}}{d_\ell}.
\end{equation}
In this paper, we use the convention that ${x}/{0}=0$ if $x=0$ and $\infty$ otherwise. This means that, if $d_\ell=0$ then a function $f\in \mathbb{H}_\ell$ belongs to the Hilbert space $\mathbb{H}_\ell^\prime$ only if $f=0$ . In such a case, $\mathbb{H}_\ell^\prime$ becomes a singleton containing only the null element of $\mathbb{H}_\ell$.
Within this framework,  $\Hbb_\ell^\prime$ is a RKHS with the kernel $\kappa_\ell'(\xb,\yb)=d_\ell \kappa_\ell(\xb,\yb)$ since \\
\begin{equation}
\begin{split}
\forall f\in \Hbb_\ell^\prime\subset \Hbb_\ell,\ f(\textbf{x})&=\left<f(\cdot),\kappa_\ell(\xb,\cdot)\right>_{\Hbb_\ell} \\
&=\frac{1}{d_\ell} \left<f(\cdot),d_\ell\kappa_\ell(\xb,\cdot)\right>_{\Hbb_\ell} \\
&=\left<f(\cdot),d_\ell\kappa_\ell(\xb,\cdot)\right>_{\Hbb_\ell^\prime}.
\end{split}	
\end{equation}

Define $\Fbb:=\Hbb_1'\times \Hbb_2' \times \cdots \times \Hbb_L'$ as the Cartesian product of the RKHSs $\{\Hbb_\ell'\}$, which is itself a Hilbert space with the inner product
\begin{equation}
\innerprod{(f_1, \dots, f_L)}{(g_1, \dots, g_L)}_\Fbb = \sum_{\ell} \innerprod{f_\ell}{g_\ell}_{\Hbb_\ell'}.
\end{equation}
Let $\Hbb:=\bigoplus_{\ell=1}^L \Hbb_\ell'$ be the \emph{direct sum} of the RKHSs $\{\Hbb_\ell'\}$, which is also a RKHS with the kernel function
\begin{equation}\label{combina_kernels}
\kappa(\xb,\yb) = \sum_{\ell} d_\ell \kappa_\ell(\xb,\yb),
\end{equation}
see \cite{aronszajn1950theory}.
Moreover, the squared norm of $f\in \Hbb$ is known as
\begin{equation}\label{sqr_norm_direct_sum}
\begin{aligned}
\|f\|^2_{\Hbb} = \min \left\{ \sum_{\ell} \|f_\ell\|^2_{\Hbb_\ell'} = \sum_{\ell} \frac{1}{d_\ell} \|f_\ell\|^2_{\Hbb_\ell} : f=\sum_{\ell} f_\ell \right. \\
\left. \ \text{such that}\ f_\ell\in\Hbb_\ell' \right\}
\end{aligned}
\end{equation}
The vector $\db=(d_1,\dots,d_L)\in\Rbb_+^L$ is seen as a tunable parameter for the linear combination of kernels $\{\kappa_\ell\}$ in \eqref{combina_kernels}.



A typical MKL task can be formulated as
\begin{subequations}\label{opt_MKL_01}
	\begin{align}
	& \underset{\substack{\fb=(f_1, \dots, f_L)\in \Fbb \\ \db\in\Rbb^L,\ b\in\Rbb}}{\min}
	& & \frac{1}{2} \sum_{\ell} \frac{1}{d_\ell} \|f_\ell\|^2_{\Hbb_\ell} + C\sum_{i} \Lcal_{0/1}(y_i, \tilde f(\xb_i)) \nonumber \\
	& \text{s.t.}
	& & d_\ell \geq 0,\ \ell\in\Nbb_L \label{d_l_simplex_01} \\ 
	& & & \sum_{\ell} d_\ell =1 \label{d_l_simplex_02} \\
	& & & \tilde f(\cdot) = \sum_{\ell} f_\ell(\cdot) + b \nonumber
	\end{align}
\end{subequations}
where $C>0$ is a regularization parameter. For our SVM task,  the first (regularization) term in the objective function is chosen so due to its convexity (see \cite[Appendix~A.1]{rakotomamonjy2008simplemkl}), which makes the problem tractable.

\section{Optimality theory}\label{sec:optimality}

In this section, we give some theoretical results on the existence of an optimal solution 
to \eqref{opt_MKL_01}, and the KKT-like first-order optimality conditions.
Our standing assumption is that each $K_\ell$ is positive definite as e.g., in the case of Gaussian kernels with different hyperparameters. We state this below formally.

\begin{assumption}\label{assump_posi_def_K_l}
	Given the data points $\{\xb_i : i\in \Nbb_m\}$, each $m\times m$ kernel matrix $K_\ell$, whose $(i, j)$ entry is $\kappa_\ell(\xb_i, \xb_j)$, is positive definite for $\ell\in\Nbb_L$.
\end{assumption}

The main results are given in the next two subsections.

\subsection{Existence of a minimizer}\label{subsec:exist}

The existence theorem is provided below with some hints of the proof.

\begin{theorem}\label{thm_exist}
	Assume that the intercept $b$ takes value from a closed interval $\Ical:=[-M, M]$ where $M>0$ is a sufficiently large number. Then the optimization problem \eqref{opt_MKL_01} has a global minimizer and the set of all global minimizers is bounded.
\end{theorem}

\begin{proof}[Sketch of the proof]
	First we appeal to the representer theorem to reduce \eqref{opt_MKL_01} to a finite-dimensional form via \eqref{multi-kernel_representer}.
	Then one can show that the sublevel sets of the objective function are compact using the fact that the step function $H$ in \eqref{func_unit_step} is lower-semicontinuous.
	Therefore, a minimizer exists by the extreme value theorem of Weierstrass.	
\end{proof}
%
\subsection{Characterization of global and local minimizers}\label{subsec:optim_cond}


The last equality constraint in \eqref{opt_MKL_01} can be safely eliminated by a substitution into the objective function. Next, define a new variable $\ub\in\Rbb^m$ by letting $u_i = 1-y_i(f(\xb_i)+b)$ where $f=\sum_{\ell} f_\ell$. 
We can then rewrite \eqref{opt_MKL_01} in the following way:
\begin{subequations}\label{opt_MKL_01.1}
	\begin{align}
	& \underset{\substack{\fb\in \Fbb,\ \db\in\Rbb^L\\ b\in\Rbb,\ \ub\in\Rbb^m}}{\min}
	& & \frac{1}{2} \sum_{\ell} \frac{1}{d_\ell} \|f_\ell\|^2_{\Hbb_\ell} + C \|\ub_+\|_0 \label{obj_func_J} \\
	& \text{s.t.}
	& & \eqref{d_l_simplex_01}\ \text{and}\ \eqref{d_l_simplex_02} \nonumber \\
	& & & u_i + y_i \left( f(\xb_i) +b \right) =1,\ i\in\Nbb_m, \label{equal_constraint_u}
	\end{align}
\end{subequations}
where the last equality constraint is obviously \emph{affine} in the ``variables'' $(\fb, b, \ub)$.

Before stating the optimality conditions, we need a generalized definition of a stationary point in nonlinear programming.

\begin{definition}[P-stationary point of \eqref{opt_MKL_01.1}]\label{def_P_stationary}
	Fix a regularization parameter $C>0$. We call $(\fb^*, \db^*, b^*, \ub^*)$ a proximal stationary $($abbreviated as P-stationary$)$ point of \eqref{opt_MKL_01.1} if there exists a vector $(\thetab^*, \alpha^*, \lambdab^*) \in \Rbb^{L+1+m}$ and a number $\gamma>0$ such that
	\begin{subequations}\label{P-stationary_cond}
		\begin{align}
		d_\ell^\ast & \ge 0, \ \ell\in\Nbb_L \label{primal_constraint_01} \\
		\sum_{\ell} d_\ell^\ast & = 1 \label{primal_constraint_02} \\
		u_i^* + y_i \left( f^*(\xb_i) +b^* \right) & =1,\ i\in\Nbb_m \label{primal_constraint_03} \\ 
		\theta_\ell^\ast & \ge 0, \ \ell\in\Nbb_L \label{dual_constraint} \\
		\theta_\ell^\ast d_\ell^\ast & = 0, \ \ell\in\Nbb_L \label{complem_slackness} \\
		\forall \ell\in\Nbb_L,\quad \frac{1}{d_\ell^\ast}f_\ell^\ast(\cdot) & =-\sum_{i}\lambda_i^* y_i \kappa_\ell(\,\cdot\,, \xb_i) \label{stationary_cond_Lagrangian_f} \\
		-\frac{1}{2(d_\ell^\ast)^2}\|f_\ell^\ast\|^2_{\Hbb_\ell}+\alpha^\ast-\theta_\ell^\ast & = 0, \ \ell\in\Nbb_L \label{stationary_cond_Lagrangian_d} \\
		\yb^\top \lambdab^* & = 0 \label{stationary_cond_Lagrangian_b} \\
		\prox_{\gamma C \|(\cdot)_+\|_0} (\ub^\ast-\gamma \lambdab^\ast) & = \ub^\ast, \label{stationary_cond_Lagrangian_u}
		\end{align}
	\end{subequations}
	where the proximal operator is defined as
	\begin{equation}\label{prox_def}
	\prox_{\gamma C\stepfunc} (\zb) := \underset{\vb\in\Rbb^m}{\argmin} \quad C\|\vb_+\|_0 + \frac{1}{2\gamma} \|\vb-\zb\|^2.
	\end{equation}
\end{definition}

According to \cite{wang2021support}, for a scalar $z$ the proximal operator in \eqref{P-stationary_cond} can be evaluated in a closed form: 
\begin{equation}\label{prox_op_scalar}
\prox_{\gamma C\stepfunc} (z) =
\left\{\begin{aligned}
0, & \quad 0<z\leq \sqrt{2\gamma C} \\
z, & \quad z>\sqrt{2\gamma C} \ \text{or} \ z\leq 0,
\end{aligned}\right.
\end{equation}
see Fig.~\ref{fig:prox_operat}.
For a vector $\zb\in\Rbb^m$, the proximal operator in \eqref{prox_def} is evaluated by applying the scalar version \eqref{prox_op_scalar} to each component of $\zb$, namely
\begin{equation}\label{prox_op_vec}
[\prox_{\gamma C\stepfunc} (\zb)]_i = \prox_{\gamma C\stepfunc} (z_i),
\end{equation}
because the objective function on the right-hand side of \eqref{prox_def} can be decomposed as
\begin{equation*}
\sum_{i} C \|(v_i)_+\|_0 + \frac{1}{2\gamma} (v_i-z_i)^2.
\end{equation*}
Formula \eqref{prox_op_vec} is called ``$L_{0/1}$ proximal operator'' in \cite{wang2021support}.

\begin{figure}
	\centering
	\begin{tikzpicture}
	\draw[line width=0.5pt][->](-1.8,0)--(2,0)node[left,below,font=\tiny]{$z$};
	\draw[line width=0.5pt][->](0,-1.5)--(0,2);
	\node[below,font=\tiny] at (0.09,0.1){0};
	\draw[color=red,thick,smooth][-](-1.2,-1.2)--(0,0);
	\draw[color=red,thick,smooth][-](0,0)--(0.8,0);
	\draw[color=red,fill=red,smooth](0.8,0)circle(0.03);
	\node[left,below,font=\tiny]at(0.8,0){$_{\sqrt{2\gamma C}}$};
	\draw[color=red][dashed] (0.8,0)--(0.8,0.8); 
	\draw[color=red,thick,smooth][-](0.8,0.8)--(1.8,1.8);
	\node[above,font=\tiny] at(1.2,1.7) {$\quad\mathrm{prox}_{\gamma C\|(\cdot)_+\|_0}(z)$};
	\end{tikzpicture}
	\caption{The $L_{0/1}$ proximal operator on the real line.}
	\label{fig:prox_operat}
\end{figure}
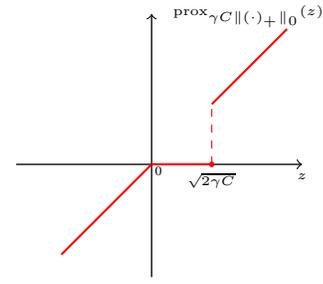

The components of the vector $(\thetab^*, \alpha^*, \lambdab^*)$ in Definition~\ref{def_P_stationary} can be interpreted as the \emph{Lagrange multipliers} as it appeared in a smooth SVM problem and played a similar role in the optimality conditions \cite{cortes1995support}, although a direct dual analysis here can be difficult due the presence of the nonsmooth nonconvex function $\|(\cdot)_+\|_0$.
The set of equations \eqref{P-stationary_cond} are understood as the \emph{KKT-like} optimality conditions for the optimization problem \eqref{opt_MKL_01.1}, where \eqref{primal_constraint_01}, \eqref{primal_constraint_02}, and \eqref{primal_constraint_03} are the primal constraints, \eqref{dual_constraint} the dual constraints, \eqref{complem_slackness} the complementary slackness, and \eqref{stationary_cond_Lagrangian_f}, \eqref{stationary_cond_Lagrangian_d}, \eqref{stationary_cond_Lagrangian_b}, and \eqref{stationary_cond_Lagrangian_u} are the stationarity conditions of the Lagrangian with respect to the primal variables. Notice that the only nonsmooth term presented is $\|\ub_+\|_0$, and the corresponding stationarity condition \eqref{stationary_cond_Lagrangian_u} with respect to $\ub$ is given by the proximal operator \eqref{prox_def}.

The following theorem connects the optimality conditions for \eqref{opt_MKL_01.1} to P-stationary points. The proof is rather lengthy and is omitted due to the space limitation.

\begin{theorem}\label{thm_optimality}
	The global and local minimizers of \eqref{opt_MKL_01.1} admit the following characterizations:	
	\begin{enumerate}
		\item A global minimizer is a P-stationary point with $0<\gamma< C_1 $,
		where the positive number
		\begin{equation*}
		C_1 = \min\left\{\lambda_{\min}(\Kcal(\db)) : \db\ \text{satisfies} \ \eqref{d_l_simplex_01}\ \text{and}\ \eqref{d_l_simplex_02}\right\}
		\end{equation*}
		in which $\lambda_{\min}(\cdot)$ denotes the smallest eigenvalue of a matrix.
		
		\item Any P-stationary point $($with $\gamma>0$$)$ is also a local minimizer of \eqref{opt_MKL_01.1}.
	\end{enumerate}
\end{theorem}

\section{Algorithm design}\label{sec:algorithm}



In this section, we take advantages of the Alternating Direction Method of Multipliers (ADMM) and working sets (active sets) to devise a first-order algorithm for our MKL-$L_{0/1}$-SVM optimization problem. More precisely, we aim to obtain a P-stationary point (Definition~\ref{def_P_stationary}) and hence  a local minimizer of \eqref{opt_MKL_01.1} by Theorem~\ref{thm_optimality}. 
The unique issue is that each $f_\ell\in\Hbb_\ell$ could be an infinite-dimensional object. Appealing to the KKT-like conditions \eqref{P-stationary_cond}, the next lemma says that $f_\ell$ admits another finite-dimensional representation which is suitable for numerical computation.

\begin{lemma}\label{lem_repres_f_l}
    Under Assumption~\ref{assump_posi_def_K_l}, each $f_\ell$ is completely represented by its values at the data points which are collected into a vector 
	\begin{equation}\label{value_vec_f_l}
	\vb_{f_\ell} :=\bmat	f_\ell(\xb_1) & f_\ell(\xb_2) & \cdots & f_\ell(\xb_m)\emat^\top.
	\end{equation}
	Moreover, we have
	$\|f_\ell\|^2_{\Hbb_\ell}=\vb_{f_\ell}^\top K^{-1}_\ell \vb_{f_\ell}$.
\end{lemma}

\begin{proof}
	According to \eqref{stationary_cond_Lagrangian_f}, we see that optimal $f_\ell$ is restricted to the linear span of the $m$ kernel sections $\{\kappa_\ell(\,\cdot\,,\xb_i)\}$. Let the coefficients of linear combination be $\{\tilde{w}_i\}$. Then it holds that $\vb_{f_\ell}=K_\ell\tilde{\wb}$ and $\|f_\ell\|^2_{\Hbb_\ell}=\tilde{\wb}^\top K_\ell \tilde{\wb} = \vb_{f_\ell}^\top K^{-1}_\ell \vb_{f_\ell}$ where $K_\ell$ is invertible by assumption.
\end{proof}

Now we can introduce the ``working set'' (with respect to the data) and related \emph{support vectors} along the lines of \cite[Subsec.~4.1]{wang2021support}. Let $(\fb^*, \db^*, b^*, \ub^*)$ be a P-stationary point of \eqref{opt_MKL_01.1}. Then by Definition~\ref{def_P_stationary}, there exist a Lagrange multiplier $\lambdab^*\in\Rbb^m$ and a scalar $\gamma>0$ such that \eqref{stationary_cond_Lagrangian_u} holds. Define a set
\begin{equation}\label{work_set_T*}
T_*:=\left\{i\in \mathbb{N}_m : u_i^*-\gamma\lambda_i^* \in (0,\sqrt{2\gamma C}\,]\right\},
\end{equation}
and its complement $\overline{T}_*:=\Nbb_m\backslash T_*$. For a vector $\zb\in\Rbb^m$ and an index set $T\subset\Nbb_m$ with cardinality $|T|$, we write $\zb_T$ for the $|T|$-dimensional subvector of $\zb$ whose components are indexed in $T$. Then it follows from \eqref{stationary_cond_Lagrangian_u}, \eqref{prox_op_vec}, and \eqref{prox_op_scalar} that 
\begin{equation}\label{value_u*}
\bmat \ub^*_{T_*} \\ \ub^*_{\overline{T}_*}\emat = \bmat \zerob \\ (\ub^*-\gamma\lambdab^*)_{\overline{T}_*}\emat.
\end{equation}
Consequently, we have $\lambdab^*_{\overline{T}_*}=\zerob$ and the working set \eqref{work_set_T*} can be equivalently written as
\begin{equation}
T_*:=\left\{i\in \mathbb{N}_m : \lambda_i^* \in [ -\sqrt{2C/\gamma}, 0)\right\}.
\end{equation}
In plain words, the nonzero components of $\lambdab^*$ are indexed only in $T_*$ with values in the interval $[ -\sqrt{2C/\gamma}, 0)$. This brings significant sparsification of the decision function since by \eqref{stationary_cond_Lagrangian_f} we have
\begin{equation}\label{f_l_via_T}
\frac{1}{d_\ell^\ast}f_\ell^\ast(\cdot) =-\sum_{i\in T_*}\lambda_i^* y_i \kappa_\ell(\,\cdot\,, \xb_i),\quad \ell\in\Nbb_L.
\end{equation}
This familiar formula calls for the following comments:

\begin{itemize}
	\item The vectors $\{\xb_i : i\in T_*\}$ correspond to nonzero Lagrange multipliers $\{\lambda_i^*\}$ just like standard support vectors in \cite{cortes1995support}. They are called \emph{$L_{0/1}$-support vectors} in \cite{wang2021support} since they are selected by the proximal operator \eqref{prox_def}.
	
	\item Moreover, the condition \eqref{primal_constraint_03} implies that 
	\begin{equation}
	y_i \left( f^*(\xb_i) +b^* \right) =1 \quad \text{for}\quad i\in T_* 
	\end{equation}
	since $\ub^*_{T_*}=\zerob$ by \eqref{value_u*}. This means that any $L_{0/1}$-support vector $\xb_i$ satisfies the equation 
	$f^*(\xb) +b^* =\pm 1$ 
	which gives two hyperplanes in the RKHS $\Hbb$, called \emph{support hyperplanes}. It is well known that such a property is guaranteed for linearly separable datasets, and may not hold for linearly inseparable datasets with penalty functions other than the $(0, 1)$-type. Such an observation explains why the $L_{0/1}$-SVM can (in principle) have fewer support vectors than other SVM models. 
\end{itemize}


Next we give the framework of ADMM for the problem \eqref{opt_MKL_01.1} which now viewed as finite-dimensional. In order to handle the inequality constraints \eqref{d_l_simplex_01}, we employ the indicator function (in the sense of Convex Analysis) 
\begin{equation}
g(\zb) = \left\{\begin{aligned}
0, & \quad \zb\in\Rbb_+^L \\
+\infty, & \quad \zb\notin\Rbb_+^L
\end{aligned}\right.
\end{equation}
of the nonnegative orthant $\Rbb_+^L$ and convert \eqref{opt_MKL_01.1} to the form:
\begin{subequations}\label{ADMM-form}
	\begin{align}
	& \underset{\substack{\fb\in \Fbb,\ \db\in\Rbb^L\\ b\in\Rbb,\ \ub\in\Rbb^m \\ \zb\in\Rbb^L}}{\min}
	& & \frac{1}{2} \sum_{\ell} \frac{1}{d_\ell} \|f_\ell\|^2_{\Hbb_\ell} + C \|\ub_+\|_0 + g(\zb)\\
	& \text{s.t.}
	& & \db = \zb \label{constaint_indica_z}\\
	& & & \eqref{d_l_simplex_02}\ \text{and}\ \eqref{equal_constraint_u}, \nonumber
	\end{align}
\end{subequations}
see \cite[Section~5]{boyd2011distributed}.
Obviously we have $g(\zb) = \sum_{\ell}g_\ell(z_\ell)$ where each $g_\ell$ is the respective indicator function for the nonnegative semiaxis $z_\ell\geq 0$.
The \emph{augmented Lagrangian} of \eqref{ADMM-form} is given by 
\begin{equation}
\begin{aligned}
 & \Lcal_{\rhob}(\fb,\db,b,\ub,\zb;\lambdab,\thetab,\alpha) = \ \frac{1}{2} \sum_{\ell} \frac{1}{d_\ell} \|f_\ell\|^2_{\Hbb_\ell} + C \|\ub_+\|_0 \\
 & \qquad + g(\zb) +\lambdab^\top \rb +\frac{\rho_1}{2}\|\rb\|^2 + \thetab^\top (\db-\zb)+\frac{\rho_2}{2} \|\db-\zb\|^2 \\
 & \qquad +\alpha\left(\oneb^\top\db-1\right)+\frac{\rho_3}{2}\left(\oneb^\top\db-1\right)^2
\end{aligned}
\end{equation}
where $\lambdab=(\lambda_1, \cdots, \lambda_m)$, $\thetab=(\theta_1, \cdots, \theta_L)$, and $\alpha$ are the Lagrangian multipliers, $\rb:=\ub + D_{\yb}\vb_{f} +b\yb-\oneb$ is the residual vector, and $\rhob=(\rho_1, \rho_2, \rho_3)$ contains three positive penalty parameters.
We have also written $f=\sum_{\ell} f_\ell$ in the 
the style of \eqref{equal_constraint_u} to simplify the notation.

Given the $k$-th iterate $(\fb^k, \db^k, b^k, \ub^k, \zb^k; \lambdab^k, \thetab^k, \alpha^k)$, the framework to update each variable is as follows:
\begin{equation}\label{ADMM-main}
\begin{array}{l}
\ub^{k+1} =  \underset{\substack{\ub\in\Rbb^L}}{\argmin}\ \Lcal_{\rhob}(\fb^k,\db^k,b^k,\ub,\zb^k; \lambdab^k,\thetab^k,\alpha^k) \\
\fb^{k+1} = \underset{\substack{\fb\in\Fbb}}{\argmin}\ \Lcal_{\rhob}(\fb,\db^k,b^k,\ub^{k+1},\zb^k; \lambdab^k,\thetab^k,\alpha^k)\\ 
b^{k+1} = \underset{\substack{b\in\Rbb}}{\argmin}\ \Lcal_{\rhob}(\fb^{k+1},\db^k,b,\ub^{k+1},\zb^k; \lambdab^k,\thetab^k,\alpha^k)\\
\zb^{k+1} = \underset{\substack{\zb\in\Rbb^L}}{\argmin}\ \Lcal_{\rhob}(\fb^{k+1},\db^k,b^{k+1},\ub^{k+1},\zb; \lambdab^k,\thetab^k,\alpha^k)\\
\db^{k+1} = \underset{\substack{\db\in\Rbb^L}}{\argmin}\ \Lcal_{\rhob}(\fb^{k+1},\db,b^{k+1},\ub^{k+1},\zb^{k+1}; \lambdab^{k},\thetab^k,\alpha^k)\\
\thetab^{k+1} = \thetab^k+\rho_2(\db^{k+1}-\zb^{k+1})\\
\alpha^{k+1} = \alpha^k+\rho_3\left(\oneb^\top \db^{k+1}-1\right) \\
\lambda_i^{k+1} = \lambda_i^k+\rho_1[u_i^{k+1}+y_i(f^{k+1}(\xb_i) +b^{k+1})-1], \ i\in\Nbb_m. \\
\end{array}
\end{equation}
Next, we describe how to solve each subproblem above. 

\begin{enumerate}
	\item \textbf{Updating} $\ub^{k+1}$. For each component of $\ub$, the $u_i$-subproblem in \eqref{ADMM-main} admits a separation of variables and can be solved along the following lines:
	\begin{equation}\label{update_u_prox}
	\begin{aligned}
	u_i^{k+1} & = \underset{u_i\in\Rbb}{\argmin} \ \ C \|\ub_+\|_0+\sum_{i}\lambda_{i}^k u_i \\
	& \quad +\frac{\rho_1}{2}\sum_{i} [ u_i + y_i (f^k(\xb_i) +b^k) -1]^2\\
	& = \underset{u_i\in\Rbb}{\argmin} \ \ C \|(u_{i})_+\|_0+\frac{\rho_1}{2}(u_i-s_i^k)^2\\
	& = \prox_{\frac{C}{\rho_1}\|(\cdot)_+\|_0}(s_i^k),
	\end{aligned}	
	\end{equation}
	where $s_i^k=1-y_i (f^k(\xb_i) +b^k) -\lambda_i^k/\rho_1$ and the proximal operator given in \eqref{prox_op_scalar}. The corresponding vector can be written compactly as
	\begin{equation}\label{vec_s^k}
	\sbf^k=\textbf{1}-D_{\yb}\vb_{f^k} -b^k\yb -\lambdab^k/\rho_1
	\end{equation}
	with $\vb_{f^k}$ defined similarly to \eqref{value_vec_f_l}.
	Define a working set $T_k$ at the $k$-th step  by 
	\begin{equation}\label{work_set_T_k}
	T_k:=\left\{i\in \mathbb{N}_m:s_i^{k}\in (0,\sqrt{2C/\rho_1}\,]\right\}.
	\end{equation}
	Then \eqref{update_u_prox} can equivalently be written as
	\begin{equation}\label{update_u}
	\ub_{T_k}^{k+1}=\zerob,\quad\ub_{\overline{T}_k}^{k+1}=\sbf^k_{\overline{T}_k}.
	\end{equation}


	\item \textbf{Updating} $\fb^{k+1}$.
	The $\fb$-subproblem in \eqref{ADMM-main} is
	\begin{equation}\label{f-solution}
	\begin{aligned}
	\fb^{k+1} = & \underset{\fb\in\Fbb}{\argmin}\ \
	\frac{1}{2} \sum_{\ell} \frac{1}{d^k_\ell} \|f_\ell\|^2_{\Hbb_\ell} + \\
	& +\sum_{i}\lambda_{i}^k \left[ u_i^{k+1} + y_i \left(f(\xb_i) +b^k \right) -1\right]\\
	& +\frac{\rho_1}{2}\sum_{i} \left[ u_i^{k+1} + y_i \left(f(\xb_i) +b^k \right) -1\right]^2.\\
	\end{aligned}
	\end{equation}
	To solve \eqref{f-solution}, we again adopt a componentwise strategy, that is, for each $d_\ell>0$ we update each $f_\ell$ separately according to the stationarity condition.
	By Lemma~\ref{lem_repres_f_l}, we only need to update the function values at all the inputs $\{\xb_i\}$, namely the vector in \eqref{value_vec_f_l}.
	More precisely, notice that by the reproducing property, the Frech\'{e}t derivative of $f_\ell(\xb_i)=\innerprod{f_\ell}{\kappa_\ell(\,\cdot\,, \xb_i)}$ with respect to $f_\ell$ can be identified as $\kappa_\ell(\,\cdot\,, \xb_i)$. Then the stationary-point equation for $f_\ell$ can be written as
	\begin{equation}\label{stationa_eqn_f_l}
	\begin{aligned}
	0= & \frac{1}{d^k_\ell}f_\ell(\cdot) +\sum_{i} \left\{\lambda_{i}^k y_i \kappa_\ell(\,\cdot\,, \xb_i) \right.\\
	& \left.+\rho_1y_i \kappa_\ell(\,\cdot\,, \xb_i)\left[ u_i^{k+1} + y_i (f(\xb_i) +b^k) -1\right]\right\}
	\end{aligned}
	\end{equation}
	which holds for any $\xb$, and in particular for all $\xb_i$. Notice that here $f=f_\ell+\sum_{t\neq\ell}f_t^k$. In vector form, we have the equation
	\begin{equation}
	\begin{aligned}
	& \left(\frac{1}{d_\ell^k}I+\rho_1K_\ell\right) \vb_{f_\ell} = \\ 
	& - \rho_1 K_\ell D_{\yb}\left(\ub^{k+1}+D_{\yb}\sum_{t\neq\ell} \vb_{f_t^k} +b^k\yb-\oneb + \lambdab^k/\rho_1\right).
	\end{aligned}
	\end{equation}
	which can be readily solved since the coefficient matrix before $\vb_{f_\ell}$ is positive definite. On the other hand, for each $d_\ell = 0$ we keep $f_\ell \equiv 0$.

	\item \textbf{Updating}  $b^{k+1}$. The $b$-subproblem can be written as
	\begin{equation}
	\begin{aligned}
	b^{k+1}= &\ \underset{\substack{b\in\Rbb}}{\argmin}\ \sum_{i}\lambda_{i}^k [u_i^{k+1} + y_i(f^{k+1}(\xb_i) +b) -1]\\
	 & \quad+\frac{\rho_1}{2}\sum_{i}[ u_i^{k+1} + y_i(f^{k+1}(\xb_i) +b) -1]^2.
	\end{aligned}
	\end{equation}
	The stationary-point equation is
	\begin{equation}\label{stationa_eqn_b}
	0=\sum_{i}\lambda^k_iy_i+\rho_1\sum_{i}y_i [u_i^{k+1} + y_i(f^{k+1}(\xb_i) +b) -1], 
	\end{equation}
	which is solved by
	\begin{equation}\label{update_b}
	b^{k+1}
	= \frac{1}{m} \left[\yb^\top (\oneb-\ub^{k+1}-\lambdab^k/\rho_1)-\oneb^\top\vb_{f^{k+1}}\right]. 
	\end{equation}
	

	\item \textbf{Updating}  $\zb^{k+1}$. The subproblem for each component of $\zb^{k+1}$ in \eqref{ADMM-main} also admits a separation of variables and we carry out the update as follows:
	\begin{equation}
	\begin{aligned}
	z_\ell^{k+1}&=\underset{z_\ell\in\Rbb}{\argmin}\ \ g(\zb) -\thetab^k\zb +\frac{\rho_2}{2} \|\db^k-\zb\|^2 \\
	&=\underset{\substack{z_\ell\in\Rbb}}{\argmin}\ \ g_\ell(z_\ell) -\theta^k_\ell z_\ell +\frac{\rho_2}{2} (d^k_\ell-z_\ell)^2\\
	&= \left(d_\ell^{k}+{\theta_{\ell}^k}/{\rho_2}\right)_+. 
	\end{aligned}
	\end{equation}
	where the function $(\,\cdot\,)_+$ takes the positive part of the argument. Inspired by this expression, we can define another working set 
	\begin{equation}\label{work_set_S_k}
	S_k := \{\ell\in\Nbb_L : d_\ell^{k}+{\theta_{\ell}^k}/{\rho_2}>0 \}
	\end{equation}
	for the selection of kernels and $\overline{S}_k=\Nbb_L\backslash S_k$. Then an equivalent update formula for $\zb$ is
	\begin{equation}\label{update_z}
	\zb^{k+1}_{S_k} = (\db^k+\thetab^k/\rho_2)_{S_k}, \quad \zb^{k+1}_{\overline{S}_k} = \zerob.
	\end{equation}
	Notice that this working set is less complicated than $T_k$ in \eqref{work_set_T_k} since the function $(\,\cdot\,)_+$ is continuous, unlike the proximal mapping.
	
	

	\item \textbf{Updating}  $\db^{k+1}$. Again we adopt a componentwise strategy for the update of $d^{k+1}_\ell$ in \eqref{ADMM-main}:
	
	\begin{equation}\label{subprob_d_l}
	\begin{aligned}
	& d^{k+1}_\ell =\underset{d_\ell\in\Rbb}{\argmin}\ \ \frac{1}{2d_\ell} \|f^{k+1}_\ell\|^2_{\Hbb_\ell} + \theta^k_\ell (d_\ell-z_\ell^{k+1})\\ 
	& + \frac{\rho_2}{2} (d_\ell-z_\ell^{k+1})^2 +\alpha^k d_\ell+\frac{\rho_3}{2} \left(d_\ell+ \sum_{t\neq\ell}d_\ell^k-1\right)^2  \\
	\end{aligned}
	\end{equation}
	where $\|f_\ell\|^2_{\Hbb_\ell} = \vb_{f_\ell}^\top K^{-1}_\ell \vb_{f_\ell}$ (see Lemma~\ref{lem_repres_f_l}),
	and $d_t^k$ are held fixed for $t\neq \ell$.
	The stationary-point equation for $d_\ell$ is just
	\begin{equation}\label{cub_eqn_d_l}
	\begin{aligned}
	0=&-\frac{1}{2d^2_\ell} \vb_{f^{k+1}_\ell}^\top K^{-1}_\ell \vb_{f_\ell^{k+1}}+\theta_{\ell}^k+\rho_2(d_\ell-z^{k+1}_\ell)\\
	&+\alpha^k+\rho_3\left(d_\ell+ \sum_{t\neq\ell}d_\ell^k-1\right).
	\end{aligned}
	\end{equation}
	This is a cubic polynomial equation (after multiplying both sides by $d_\ell^2$) which can be solved numerically. Moreover, since the coefficients are real, there must be at least one real root. When $\|f^{k+1}_\ell\|^2_{\Hbb_\ell}>0$, the objective function in \eqref{subprob_d_l} is strictly convex in the positive semiaxis $d_\ell>0$ and one can show without difficulty that a local minimizer, which must be a stationary point, exists. Therefore, we conclude that \eqref{cub_eqn_d_l} must have a \emph{positive real root}.

	However, if instead we search the minimum of \eqref{subprob_d_l} in all of $\Rbb$, a pathology happens when $\|f^{k+1}_\ell\|^2_{\Hbb_\ell}>0$ and $d_\ell$ tends to zero from the left. In that case, the objective function tends to $-\infty$ and $+\infty$ on two sides of zero.
	As an ad-hoc recipe, we take the positive real root\footnote{If there are multiple positive real roots (possibly due to some numerical issue), we take the one with the largest absolute value.} of \eqref{cub_eqn_d_l} as $d_\ell^{k+1}$ for $\ell\in S_k$. For $\ell\in \overline{S}_k$,  on the other hand, we let $\db^{k+1}_{\overline{S}_k}=\zerob$
	in accordance with the second formula in \eqref{update_z}, because in the end (if the algorithm converges) we will have the equality \eqref{constaint_indica_z}.
	
	
	
	
	\item \textbf{Updating}  $\thetab^{k+1}$. With the help of the working set \eqref{work_set_S_k}, the update of $\thetab$ in \eqref{ADMM-main} can be simplified as:	
	\begin{equation}\label{update_theta}
	\thetab^{k+1}_{S_k}  =  \thetab^k_{S_k} + \rho_2(\db^{k+1}-\zb^{k+1})_{S_k},\quad \thetab^{k+1}_{\overline{S}_k}  =  \thetab^k_{\overline{S}_k}.
	\end{equation}
	

	\item \textbf{Updating}  $\alpha^{k+1}$. See \eqref{ADMM-main}. 
	

	\item \textbf{Updating}  $\lambdab^{k+1}$. Inspired by the property of the working set $T_*$ (see \eqref{value_u*} and the next line after it), the update of $\lambdab$ in \eqref{ADMM-main} is simplified as follows:
	
	\begin{equation}\label{update_lambda}
	\begin{aligned}
	\lambdab^{k+1}_{T_k} = \lambdab^{k}_{T_k}+\rho_1\rb^{k+1}_{T_k},\quad \lambdab^{k+1}_{\overline{T}_k}=\zerob 
	\end{aligned}
	\end{equation}
	where the vector $\rb^{k+1}=\ub^{k+1} + D_{\yb}\vb_{f^{k+1}} +b^{k+1}\yb-\oneb$. In other words, we remove the components of $\lambdab$ which are not in the current working set.
\end{enumerate}

The update steps above are collected into the next algorithm.

\begin{algorithm}
	\caption{ADMM for the MKL-$L_{0/1}$-SVM}
	\label{alg:admm_MKL_SVM}
	\begin{algorithmic}[1]
		\State Set $C$, $\rho_1$, $\rho_2$, $\rho_3$, $\{\kappa_\ell\}$, \texttt{max\_iter}, and 
		 $k=0$.
		\State Initialize $(\fb^0, \db^0, b^0, \ub^0, \zb^0; \thetab^0,\alpha^0,\lambdab^0)$.
		\While{The terminating condition is not met and $k\le \texttt{max\_iter}$}
		\State Update                                                                                                                                                                                                                                           $T_k$ as in \eqref{work_set_T_k}. 
		\State Update $\ub^{k+1}$ by \eqref{update_u}.
		\State Update $\fb^{k+1}$ by \eqref{stationa_eqn_f_l}.
	     \State Update $b^{k+1}$ by \eqref{update_b}.
		\State Update $\zb^{k+1}$ by \eqref{update_z}.
		\State Update $\db^{k+1}$ by \eqref{cub_eqn_d_l}.
		\State Update $\thetab^{k+1}$ by \eqref{update_theta}.
		\State Update $\alpha^{k+1}$ in \eqref{ADMM-main}.
	  \State Update $\lambdab^{k+1}$ by \eqref{update_lambda}.
	  \State Set $k=k+1$.
 		\EndWhile\\
 		\Return the final iterate $(\fb^k,\db^k,b^k,\ub^k,\zb^k; \thetab^k,\alpha^k,\lambdab^k)$.
	\end{algorithmic}\label{ADMM-Solver}
\end{algorithm}
%
%
%
%


Unfortunately, we are not able to prove the convergence of Algorithm~\ref{alg:admm_MKL_SVM} as it seems very hard in general due to the nonconvexity and nonsmoothness of the optimization problem \eqref{opt_MKL_01.1}. However, we can give a characterization of the limit point \emph{if} the algorithm converges, see the next result.

\begin{theorem}\label{thm_converg_implies_loc_opt}
	Suppose that the sequence 
	\begin{equation*}
	\{\Psi^k\} = \{(\fb^k, \db^k, b^k, \ub^k, \zb^k; \thetab^k, \alpha^k, \lambdab^k)\}
	\end{equation*}
	generated by the ADMM algorithm above has a limit point $\Psi^*=(\fb^*, \db^*, b^*, \ub^*, \zb^*; \thetab^*, \alpha^*, \lambdab^*)$. Then $(\fb^*, \db^*, b^*, \ub^*)$ is a P-stationary point with $\gamma=1/\rho_1$ and also a local minimizer of the problem \eqref{opt_MKL_01.1}.
\end{theorem}
\begin{proof}[Sketch of the proof]
The idea is to 
use the convergence properties of the subproblems. It can be shown that the limit point obtained by the ADMM algorithm is a P-stationary point whose local optimality is guaranteed by Theorem~\ref{thm_optimality}.
\end{proof}
\begin{remark}\label{rem:sparsity_kernel_combi}
	Similar to the working set $T_*$ in \eqref{work_set_T*} and the associated support vectors, the working set $S_k$ in \eqref{work_set_S_k} and its limit set $S_*$ renders \emph{sparsity} in the combination of the kernels $\{\kappa_\ell\}$ for the MKL task, because the constraint \eqref{d_l_simplex_02} can be interpreted as $\|\db\|_1=1$, an equality involving the $\ell_1$-norm, due to the nonnegativity condition \eqref{d_l_simplex_01}. Such an effect of sparsification can also be observed from our numerical example in the next section.
\end{remark}

%

\section{Simulation on synthetic data}\label{sec:sims}



In this section, we conduct numerical experiments using Matlab on a Dell laptop workstation with 64GB of memory and an Intel Core i7 2.5GHz CPU on synthetic data to demonstrate the sparsity and effectiveness of the proposed MKL-$L_{0/1}$-SVM. 
The simulation presented here is very simple and by no means extensive. More precisely, we work with two-dimensional (planar) data ($n=2$) for the purpose of easy visualization using a ten-kernel ($L=10$) \LzeroneSVM. The ten kernel functions are all Gaussian, see \eqref{Gaussian_kernel}, with quite arbitrarily chosen hyperparameters $\{\sigma_\ell\}$ which are listed in Table~\ref{table:hyperparameters}. Some specific points are discussed next.

\begin{table}[h]
	\begin{center}
		\caption{An arbitrary choice of the hyperparameters}
		\label{table:hyperparameters}
		\begin{tabular}{ccccc}
			\toprule
			$\sigma_1$&$\sigma_2$&$\sigma_3$&$\sigma_4$&$\sigma_5$ \\
			\midrule
			$0.1400$&$0.0995$&$0.0161$&$0.0409$&$0.1561$ \\
			\midrule
			$\sigma_6$&$\sigma_7$&$\sigma_8$&$\sigma_9$&$\sigma_{10}$ \\
			\midrule
			$0.0156$&$0.1221$&$0.1175$&$0.0539$&$0.1247$ \\
			\bottomrule
		\end{tabular}
		
	\end{center}
	
\end{table}

\textbf{(a) Stopping criteria.} In the implementation, we terminate Algorithm~\ref{alg:admm_MKL_SVM} if the iterate $(\fb^k$, $\db^k$  $b^k$, $\ub^k$, $\zb^k$; $\thetab^k$, $\alpha^k$, $\lambdab^k)$ satisfies the condition:
 \begin{equation}\label{stopping_criteria}
\max \left\{\beta_1^k,\beta_2^k,\beta_3^k,\beta_4^k,\beta_5^k,\beta_6^k,\beta_7^k,\beta_8^k \right\}< \tol,
 \end{equation}
    where the number $\tol>0$ is the tolerance level and 
    \begin{equation}
    	\begin{array}{ll}
    		\beta_1^k:=\|\ub^k-\ub^{k-1}\|,\quad
    		& \beta_2^k:=\|\fb^k-\fb^{k-1}\|, \\
    			\beta_3^k:=|b^k-b^{k-1}|,\quad
    		& \beta_4^k:=\|\zb^k-\zb^{k-1}\|, \\
    			\beta_5^k:=\|\db^k-\db^{k-1}\|,\quad
    		& \beta_6^k:=\|\thetab^k-\thetab^{k-1}\|, \\
    		\beta_7^k:=|\alpha^k-\alpha^{k-1}|,\quad
    		& \beta_8^k:=\|\lambdab^k-\lambdab^{k-1}\|.
    	\end{array}    	
    \end{equation}
The condition says, in plain words, that two successive iterates are sufficiently close. 

  \textbf{(b) Parameters setting.} 
  In Algorithm~\ref{alg:admm_MKL_SVM}, the parameters $C$ and $\rho_1$ characterize the working set \eqref{work_set_T_k} which is related to the number of support vectors, see \eqref{f_l_via_T} and the comments right after it. For simplicity, we have taken $\rho=\rho_1=\rho_2=\rho_3$ in the ADMM algorithm. 
  In order to choose the two parameters, the standard $10$-fold Cross Validation (CV) is employed on the training data, where $C$ and $\rho$ are selected from $\{2^{-2},2^{-1},\cdots,2^{7}\}$ and $\{a^{-2},a^{-1},\cdots,a^{7}\}$ with $a=\sqrt{2}$, respectively.  The parameter combination with the highest CV accuracy is picked out. In addition, we set the maximum number of iterations $\texttt{max\_iter}=10^3$ in Algorithm~\ref{alg:admm_MKL_SVM} and the tolerance level $\tol=10^{-3}$ in \eqref{stopping_criteria}.
  
  For the starting point, we set $\ub^0=\lambdab^0=\zerob$, $\thetab^0=\zerob$, $\fb^0=\zerob$, $\zb^0=\zerob$, $\alpha^0=0$ and $\db^0=\frac{1}{L} \oneb$, $b^0=1$ or $-1$. 
  The reason for such a choice is explained in the following. Let us call the objective function in \eqref{opt_MKL_01} $J(\fb,\db,b)$. The we immediately notice that $J(\zerob, \frac{1}{L} \oneb, 1) = Cm_-$ and $J(\zerob, \frac{1}{L} \oneb, -1) = Cm_+$ where $m_+$ and $m_-$ denote the numbers of positive and negative components in the label vector $\yb$.
  Therefore, we should choose $(\fb^0, \db^0, b^0)$ such that $J(\fb^0, \db^0, b^0)\leq C\min\{m_+, m_-\}$.

\textbf{(c) Evaluation criteria.} To evaluate the classification performance of our \LzeroneSVM, we report two criteria: the testing accuracy (\texttt{TACC}) and the number of support vectors (\texttt{NSV}) which is equal to $|T_*|$. Let $\{(\xb^{\test}_j,y^{\test}_j):j=1,\cdots,m_{\test}\}$ be the testing data. The testing accuracy is defined as
\begin{equation*}
	\texttt{TACC}:=1-\frac{1}{2m_{\test}}\sum_{j=1}^{m_{\test}} \left|\sign\left(\sum_{\ell} f^*_\ell(\xb^{\test}_j)+b^*\right)-y^{\test}_j \right|.
\end{equation*}
Here the quantity $\sum_{\ell} f^*_\ell(\xb^{\test}_j)$ can be computed using \eqref{stationary_cond_Lagrangian_f}. More specifically, for each $\ell\in\Nbb_L$ we can evaluate $f_\ell^\ast(\cdot) = - d_\ell^* \sum_{i}\lambda_i^* y_i \kappa_\ell(\,\cdot\,, \xb_i)$ on the test data using the convergent iterate produced by Algorithm~\ref{alg:admm_MKL_SVM}.

\textbf{(d) Simulation result.} 
The planar data are generated randomly in the four quadrants and then (randomly) split into a training set and a testing set of equal size, i.e., $m=m_{\test}=100$. 
More specifically\footnote{For details about data generation, the reader can refer to the section ``Train SVM Classifier Using Custom Kernel'' in the online documentation \url{https://www.mathworks.com/help/stats/support-vector-machines-for-binary-classification.html}.}, points in the first and third quadrants are given label $1$, while points in the second and fourth quadrants are given label $-1$. 
The parameters $C$, $\rho$ are determined by the CV procedure described before on the training set.  
After that, the MKL-\LzeroneSVM\ is optimized on the training set via Algorithm~\ref{alg:admm_MKL_SVM}, and then the accuracy of the optimal classifier is verified using the testing set. The results are shown in Fig.~\ref{fig:result_classif} and Table~\ref{table:parameters} which also gives the best parameters selected during the CV. Notice that we have only reported the four $d_\ell$'s which are significant in size (larger than $10^{-4}$). The rest $d_\ell$'s are set to zero.
Additionally, we found that setting $d_1^*$ and $d_6^*$ (which are relatively small among the four) to $0$ does not affect \texttt{TACC}.
Therefore, our optimization procedure has claimed the following: the optimal linear combination of the ten candidate kernels in our \LzeroneSVM\ involves essentially only \emph{two} kernels, namely the ones with hyperparameters $\sigma_3$ and $\sigma_5$. In other words, we have obtained great sparsity in the kernel combination in accordance with Remark~\ref{rem:sparsity_kernel_combi}.
We comment at last that the testing accuracy of $90\%$ is obviously not good enough and an improvement is possible via a better selection of the kernels and the hyperparameters.




\begin{table}[h]
	\begin{center}
		\caption{Simulation results}
		\label{table:parameters}
		\begin{tabular}{cccccccc}
			\toprule
			$C$&$\rho$&$d_1^*$&$d_3^*$&$d_5^*$&$d_6^*$&\texttt{TACC}&\texttt{NSV}\\
			\midrule
			$16$&$4$&0.0021&$0.4575$&$0.5290$&$0.0113$&$0.90$&$100$ \\
			\bottomrule
		\end{tabular}
		
	\end{center}
	
\end{table}

\begin{figure}[H]
	\centering
	\includegraphics[width= 6cm]{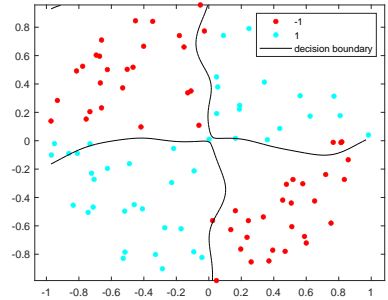}
	\caption{Scatter diagram for the result of classification on the \emph{testing} set with the decision boundary which corresponds to all $\xb$ (on a regular grid) that satisfy the equation $f^*(\xb) +b^* =0$.}
	\label{fig:result_classif}
\end{figure}


\section{Conclusions}\label{sec:conclu}

We have considered a MKL task for the $L_{0/1}$-SVM in order to select the best possible combination of some given kernel functions while minimizing a regularized $(0, 1)$-loss function. Despite the nonconvex and nonsmooth nature of the objective function, we have provided a set of KKT-like first-order optimality conditions to characterize global and local minimizers.
Numerically, we have developed an efficient ADMM solver to obtain a locally optimal solution to the MKL-\LzeroneSVM\ problem.
Preliminary simulation results have shown the effectiveness of our theory and algorithm.

\section*{Acknowledgments}

The authors would like to thank Mr. Jiahao Liu for his assistance in the Matlab implementation of Algorithm~\ref{ADMM-Solver}.


\bibliographystyle{IEEEtran}
\bibliography{references}   

\end{document}